\documentclass[10pt,aps,prl,twocolumn,showpacs,superscriptaddress]{revtex4-1}
\usepackage{amsmath,graphicx}
\usepackage{amssymb}
\usepackage{algpseudocode}
\usepackage{bm}
\usepackage{color}
\usepackage{graphicx}
\usepackage{xcolor}
\usepackage{url}

\usepackage{amsthm}
\newtheorem{mydef}{Definition}
\newtheorem{mytheo}{Theorem}

\begin{document}
\title{From line segments to more organized Gestalts}

\author{Boshra Rajaei$^{1,2}$, Rafael Grompone von Gioi $^{1}$,
      Jean-Michel Morel$^{1}$}
\noaffiliation
\affiliation{
CMLA, ENS Cachan, France\\
$^2$ Sadjad University of Technology, Mashhad, Iran
}

\begin{abstract} 
In this paper, we reconsider the early computer vision bottom-up program,
according to which higher level features (geometric structures) in an image
could be built up recursively from elementary features by simple grouping
principles coming from Gestalt theory. Taking advantage of the (recent) advances
in reliable line segment detectors, we propose three feature detectors that
constitute one step up in this bottom up pyramid. For any digital image, our
unsupervised algorithm computes three classic Gestalts from the set of
predetected line segments: good continuations, nonlocal alignments, and
bars. The methodology is based on a common stochastic {\it a contrario model}
yielding three simple detection formulas, characterized by their number of false
alarms. This detection algorithm is illustrated on several digital images.
\end{abstract}
\maketitle

\section{Introduction}\label{sec:Introdution}  

A basic task of the visual system is to group fragments of a scene into objects
and to separate one object from others and from background. Perceptual grouping
has inspired many researches in psychology in the past century.  In 1920's, a
systematic approach for investigating human perception was introduced under the
name of Gestalt theory \cite{wertheimer}. The Gestalt school proposed the
existence of a short list of grouping laws governing visual perception. The
Gestalt laws include, but are not limited to, similarity, proximity,
connectedness and good continuation. Gestaltists argue that partial grouping laws
recursively group image elements to form more organized Gestalts.  But the
grouping laws were qualitative and lacked a quantitative formalization.

Since its origins, computer vision has been interested in Gestalt laws and there
have been several attempts to formalize aspects of the Gestalt program
\cite{Lowe85,sarkar1993}. We will concentrate on a particular approach
\cite{DMM2000,dmm08} which has led to the conception of several Gestalt
detectors to organize meaningful geometric structures in a digital image. Based
on the \emph{non-accidentalness principle} \cite{Lowe85,rock-logic}, an observed
structure is considered meaningful when the relation between its parts is too
regular to be the result of an accidental arrangement of independent parts. This
is the rationale behind the \textit{a contrario} model for determining potential
Gestalts which is formulated in the next section.

The \textit{a contrario} methodology was used in \cite{lezama_align}, inspired
by the \textit{good continuation} Gestalt principle, to propose an algorithm to
detect alignments of points in a point pattern. The detected point alignments
are further employed toward a nonparametric vanishing point estimating algorithm
again based on the non-accidentalness principle \cite{lezama_vanishing}. Based
on the same principle, an automatic line segment detector (LSD) is provided in
\cite{lsd_pami,lsd_ipol} with linear time complexity. The same approach is used
in the EDLines straight line calculator \cite{akinlar2011edlines} but with the
difference that the latter algorithm starts from an edge drawing output while
the former is based on image level lines. Similar ideas were also proposed to
detect circles and ellipses \cite{elsd,akinlar2013edcircles}. Level lines are
employed to detect good continuations and image corners in \cite{cao}. Using
edge pieces (edgelets) the authors of \cite{widynski2011edgelet} adopt a
featured \textit{a contrario} scheme to extract the strong edges and eliminate
irrelevant and textural ones. Extracting an object of interest from textural
background or from outlier points is studied in
\cite{grosjean2009texture,gerogiannis2015elimination} where, for instance, fixed
contrast spots are detected in mammographical images.

The \textit{a contrario} framework was used for image segmentation in
\cite{burrus2009image}. To obtain robust segments, the authors suggest a
combination of the \textit{a contrario} approach and of Monte-Carlo
simulation. Also, to solve the hierarchical segmentation issue, the authors of
\cite{cardelino2009hierarchical} suggest two \textit{a contrario} criteria for
measuring region and merging meaningfulness based on homogeneity and boundary
contrast.

In this paper, we start from the ``partial Gestalts" constituted by all LSD
detected segments, and explore if they can be used as basic elements for more
elaborate Gestalt groups such as long straight lines, good continuations and
parallel segments. Toward this aim, the \textit{a contrario} model is adopted on
the line segment distribution. It is used repeatedly for detecting the mentioned
structures as non-accidental. We shall observe in most images that, with the
exception of a few isolated segments, all line segments are classified precisely
and hence the algorithm forms another level of segmentation pyramid toward a
complete analysis and understanding of each image's structure. The main
advantage of this approach is its low time complexity due to the exploitation of
line segments as input structures which are far less numerous than the
unstructured set of all image pixels.

\section{Theoretical modeling}\label{sec: theoretical_modeling}

The \textit{a contrario} detection approach has a probabilistic basis. Consider
an event of interest $e$. According to the non-accidentalness principle, $e$ is
meaningful if its expectation is low under the stochastic 
model $H_0$. The stochastic expectation of an event is called its number of
false alarms (NFA) and is defined as
\begin{equation}\label{equ:NFA}
\text{NFA}(e) = N_{test}P_{H_0}(e)
\end{equation}
where $N_{test}$ indicates the number of possible occurrences of $e$ and
$P_{H_0}(e)$ is the probability of $e$ happening under $H_0$. A relatively small
NFA implies a rare event $e$ under the \textit{a contrario} model and therefore,
a meaningful one. An event $e$ is called $\epsilon$-meaningful if
$\text{NFA}(e)<\epsilon$. In this paper, three types of events are under focus:
non-local alignments, good continuations and parallelism. For each of them, the
\textit{a contrario} model $H_0$ is adopted of stochastic line segments, with
independent and uniformly distributed tips on the image domain. Here, we provide
a formal definition of each event.
\begin{mydef}\label{def:good_cont}
A sequence of $k$ line segments $l_i, l_{i+1}, \dots, l_j$ form a
\textnormal{potential} good continuation $\wp^{k,d,\theta}$ if and only if
\linebreak $d = D(l_i, l_{i+1}, \dots, l_j) < \rho$ and $\theta=\angle(l_i,
l_{i+1}, \dots, l_j)<\theta_s$ for predefined constant $\rho$ and $\theta_s$
thresholds.
\end{mydef}
In the above definition $D(\cdot)$ and $\angle(\cdot)$ denote respectively the
maximum of the distances and the maximum of the angles between successive line
segments in a sequence. The two threshold values $\rho$ and $\theta_s$ limit the
search space around each line tip for finding close line segments as represented
in Figure \ref{fig:sector}. A small margin $\lambda$ is allowed to deal with
possible misalignments in LSD outputs. This special adopted sector is denoted by
$\mathcal{S}$ in the sequel.

\fboxsep=0pt%padding thickness
\fboxrule=0.5pt%border thickness
\begin{figure} 
\centering
\fbox{\includegraphics[width=.2\textwidth]{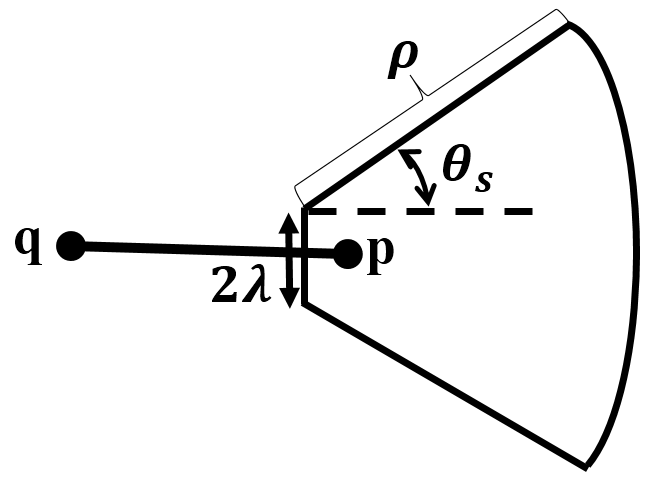}}
\caption{Given the endpoint $p$ of a line segment, this figure defines the
  domain $\mathcal S$ around $p$, in which the presence of successive segment
  endpoints is searched by the algorithm. This search space depends on three
  parameters, $\rho$, $\theta_s$ and $\lambda$.}
\label{fig:sector}
%\vspace{-0.50cm}
\end{figure}

\begin{mydef}\label{def:sline}
A good continuation satisfying $\theta_s<3^\circ$ is called a non-local alignment
event $\zeta^{k,d,\theta}$.
\end{mydef}
\begin{mydef}\label{def:parallel}
Two line segments or non-local alignments $l_i$ and $l_j$ are said to form a bar
$\Im^{d,\theta}$ if and only if $d=D_m(l_i,l_j)<\rho$ and
$\theta=\angle(l_i,l_j)\approx\pi$ for a predefined constant $\rho$ threshold.
\end{mydef}
Here, $D_m(\cdot)$ stands for the mutual distance operator which is calculated
as the average distance between the respective tips of both segments.

\begin{mytheo}\label{Theo:GC}
Under the \textnormal{a contrario} model that all segment tips are uniform
i.i.d. spatial variables in the image domain, the number of false alarms (NFA)
of a good continuation event $\wp^{k,d,\theta}$ according to Definition
\ref{def:good_cont} is equal to \\
\begin{equation}\label{equ:NFA_GC}
  \text{NFA}(\wp^{k,d,\theta})
    = 2N \cdot 3^{k-1} \cdot \left( (N-1) \frac{\theta d^2}{mn} \cdot
                                   \frac{\theta}{\pi} \right)^{k-1}
\end{equation}
where $N$ indicates the total number of line segments.
\end{mytheo}
\begin{proof}
According to (\ref{equ:NFA}), the computation of the NFA consists of two parts:
the number of tests and the probability of the geometric event under
$H_0$. Here, we have $2N$ different choices for the first tip of
$\wp^{k,d,\theta}$ (2 tips per line segment). Assuming that at each of the $k-1$ line segment
extremities we test the 3 closest tips at most, we have 3 choices per joint and, therefore, 
$N_{test}=2N3^{k-1}$.

For the probability term of $P_{H_0}(\wp^{k,d,\theta})$, let us first compute
the probability that only two line segments be in good continuation as specified
in Figure \ref{fig:sector} and with parameters $(d,\theta)$.  We shall then
extend the computation to $k$ segments. This probability
consists of two terms: $\Pi_s \simeq \frac{\theta d^2}{mn}$ is the probability
that one tip falls into $\mathcal{S}^{d,\theta,\lambda}$ around the other
tip\footnote{Since $\lambda$ is relatively small, we use this simple tight upper
  bound for $\Pi_s$.  We also neglect the fact that the angular sector may
  sometimes fall outside the image domain, with area $mn$.} and 
$\Pi_a\simeq\frac{\theta}{\pi}$ is the probability that a maximum angle of
$\theta$ occurs between the two line segments. Therefore, having one endpoint
in this area is the complement of the event where none of the $(N-1)$ other
endpoints is there. Thus $P_{H_0}(\wp^{2,d,\theta})=1-(1-\Pi_s
\Pi_a)^{N-1}$. Consequently, for $k-1$ junctions in a sequence of $k$ segments
we have $P_{H_0}(\wp^{k,d,\theta})=(1-(1-\Pi_s \Pi_a)^{N-1})^{k-1}\simeq
((N-1)\Pi_s\Pi_a)^{k-1}$.
\end{proof}
\begin{mytheo}\label{Theo:SL_PL}
Under the same \textit{a contrario} assumption as in Theorem \ref{Theo:GC}, the
number of false alarms of a non-local alignment event, $\zeta^{k,d,\theta}$, according
to Definition \ref{def:sline} is equal to \\
\begin{equation}\label{equ:NFA_SL}
  \text{NFA}(\zeta^{k,d,\theta})
    = 2N \cdot 3^{k-1} \left( (N-1) \frac{2\lambda d}{mn}
                              \cdot \frac{\theta}{\pi} \right)^{k-1},
\end{equation}
and the number of false alarms of parallel line (bar) event, $\Im^{d,\theta}$,
according to Definition \ref{def:parallel} is equal to \\
\begin{equation}\label{equ:NFA_PL}
  \text{NFA}(\Im^{d,\theta})
    = 3N \cdot (N-1) \left( \frac{\pi d^2}{mn} \right)^2
      \cdot\frac{\theta}{\pi}.
\end{equation}
\end{mytheo}
\begin{proof}
The proof of the above theorem is similar to that of Theorem
\ref{Theo:GC}. Since non-local alignments are special types of good continuation,
their NFA calculation is straightforward. The only point is that for small
$\theta<3^\circ$, $\lambda$ is not negligible anymore. Indeed, for small angles
$\mathcal{S}^{d,\theta,\lambda}$ turns into a rectangle with dimensions
$2\lambda$ by $d$. Accordingly, by replacing $\Pi_s=\frac{2\lambda d}{mn}$ in the derivation
of Theorem \ref{Theo:GC}, we obtain (\ref{equ:NFA_SL}).

Using similar reasoning for the case of parallel lines, since we only have two
segments or straight lines, the $N_{test}$ term reduces to $3N$. By substituting
$\theta=\pi$ according to Definition \ref{def:parallel}, the overall NFA is
formulated as (\ref{equ:NFA_PL}).
\end{proof}

\section{Gestalt detector}\label{sec: GD}

As mentioned earlier, by employing line segments as initial groups it is
possible to detect more organized and sophisticated Gestalts. In this paper, we
exploit the LSD algorithm \cite{lsd_pami} to produce $l_1, l_2, \dots, l_N$
initial line segments. The next step is finding all possible instances of our
three Gestalts of interest in the form of sequences of line segments or
\textit{chains}\footnote{A time efficient algorithm to prune the detection space
  can be tested in the online facility http://gestalt-detector.webs.com .}.
\begin{mydef}
A good continuation (Definition \ref{def:good_cont}) chain $\wp^{k,d,\theta}$ is
meaningful if and only if $\text{NFA}(\wp^{k,d,\theta})<\epsilon$ and there
exists no subchain of $\wp$ with smaller NFA.
\end{mydef}
Meaningful non-local alignments and bars follow the same definition. We assume
$\epsilon=1$ as a simple way to allow less than one false alarm on average for each
Gestalt type per image \cite{lsd_ipol}. Finally, by merging all meaningful
Gestalts we obtain a drawing of the input image. The next section addresses the
efficiency of the proposed approach applied on real-world images.

\section{Experimental Results} \label{sec: simulations} 

Man-made structures like buildings, furniture and natural essences present many
good continuations in the form of line segments and curves or basic geometrical
shapes like rectangles and parallelograms. To organize these structures using
the proposed algorithm in this paper, we first applied LSD algorithm
\cite{lsd_ipol}. Afterwards, the two initial thresholds $\theta_s$ and $\rho$ in
Definitions \ref{def:good_cont} to \ref{def:parallel} must be determined. These
two parameters limit the number of chains that later are considered by the
\textit{a contrario} model. The parameter $\theta_s$ controls the smoothness of
output Gestalts while $\rho$ is a parameter proportional to the image size
restricting the maximum acceptable distance between line segments of a candidate
chain. In the sequel and to detect all kinds of parallelogram shapes, $\theta_s$
is set to $150^\circ$ and $\rho$ is experimentally fixed at $\min(10, \lceil
0.1\times \max(m,n) \rceil)$ pixels by default.

Figure \ref{fig:Detailed} shows the result of the joint application of all
feature detectors on an image. In each output image, organized Gestalts are
depicted with different colors and finally the residual line segments are shown
in the last image. Note that the detected Gestalts might be grouped again into
longer continuations using another level of \textit{a contrario} model. What is
remained from the image mostly only consists of isolated line segments that are
not expected to belong to any organized structure. More results on more
real-world or synthetic images are provided in Figure \ref{fig:Several}.

\begin{figure*}
\centering
\fbox{\includegraphics[width=0.26\textwidth]{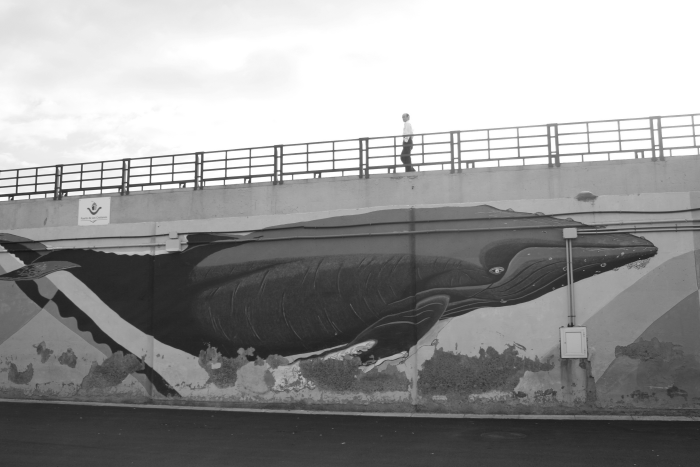}}~
\fbox{\includegraphics[width=0.26\textwidth]{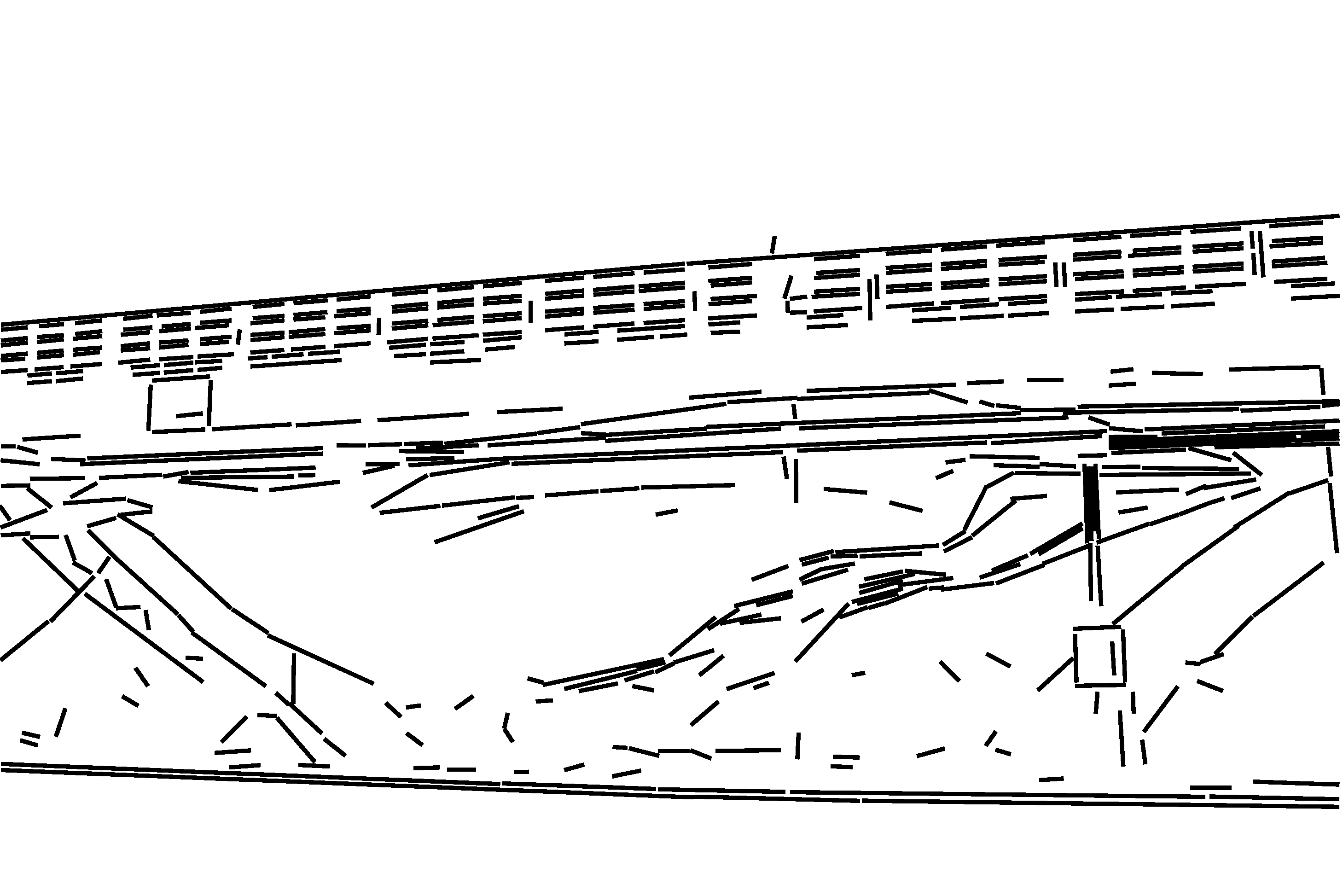}}~
\fbox{\includegraphics[width=0.26\textwidth]{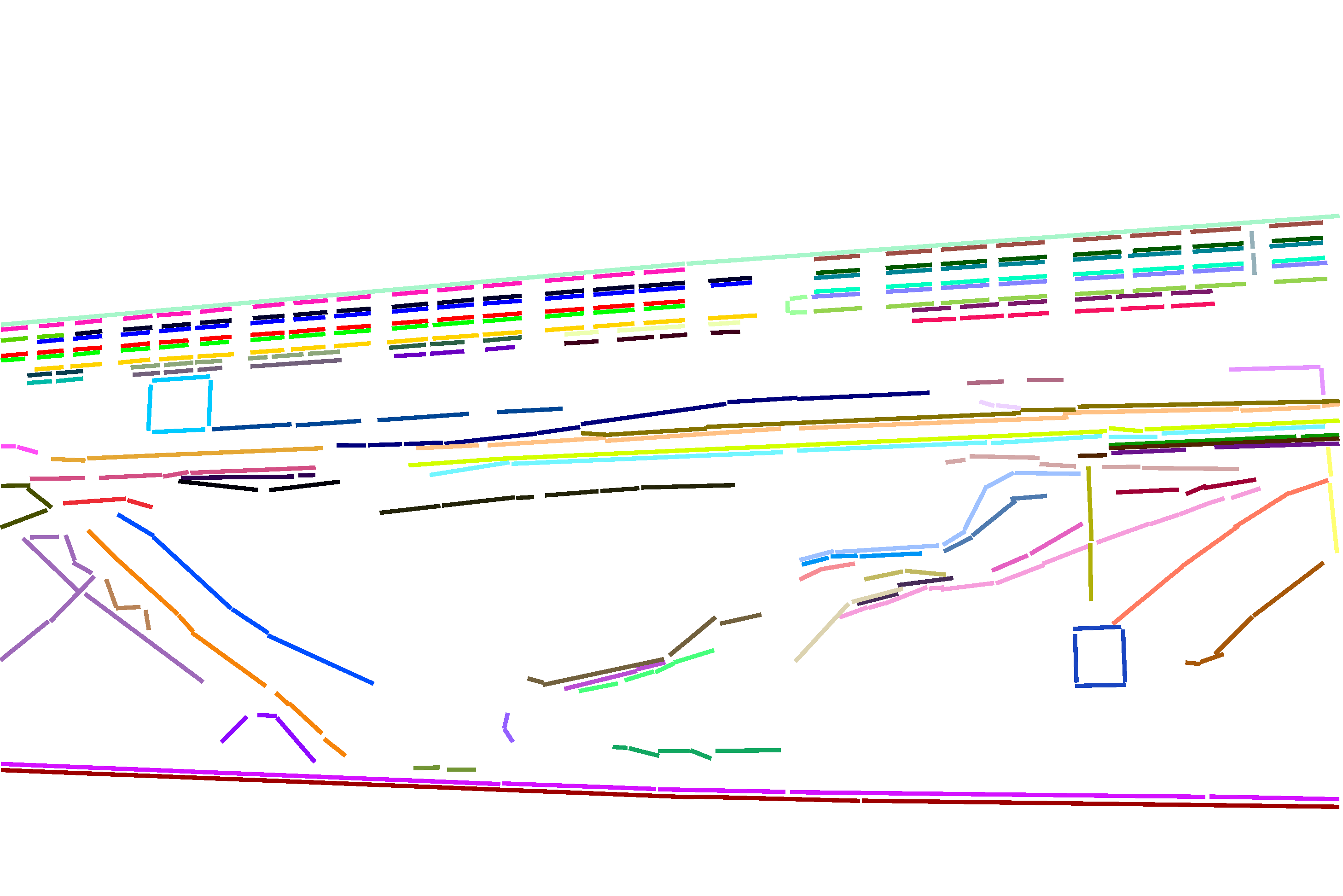}}
\fbox{\includegraphics[width=0.26\textwidth]{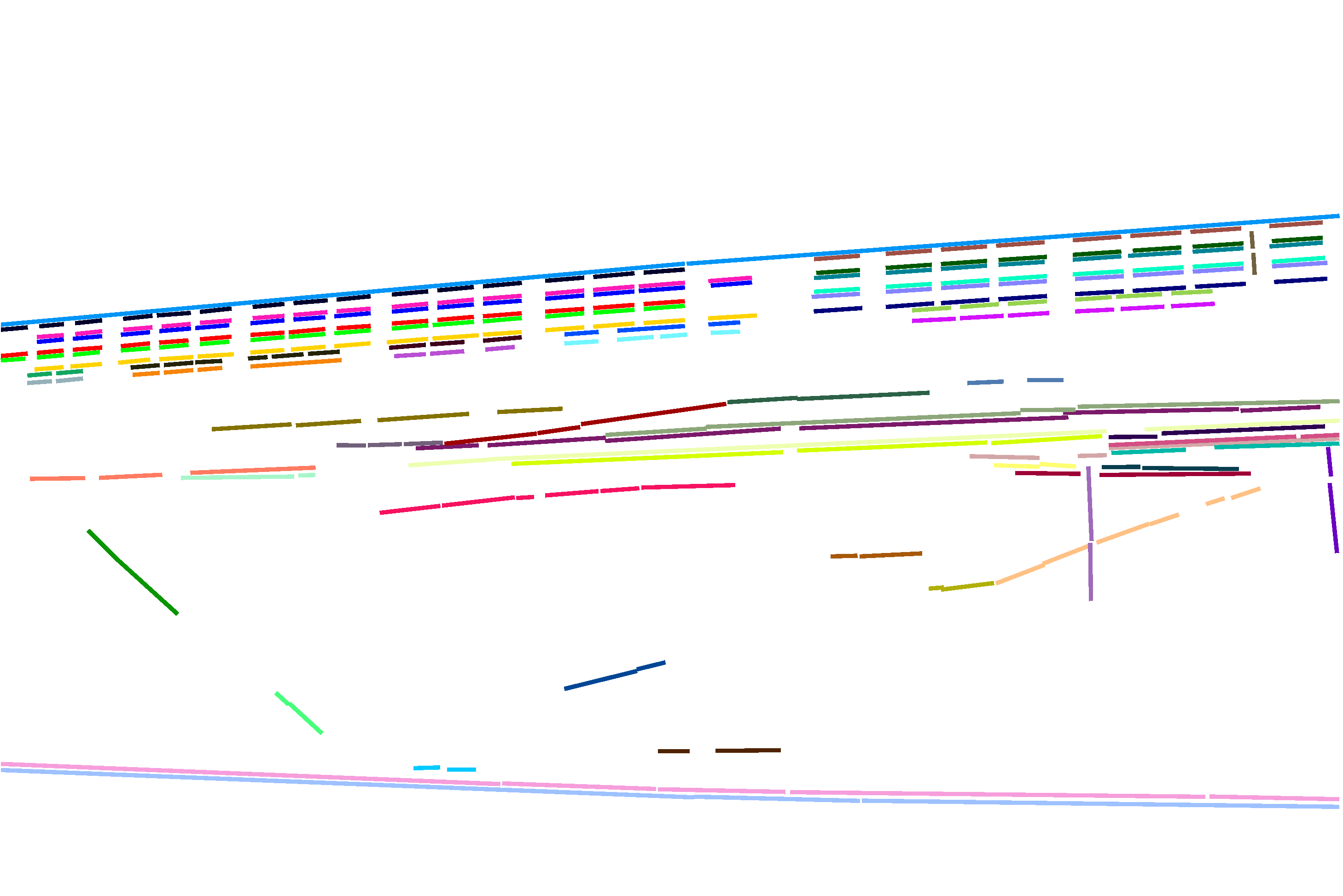}}~
\fbox{\includegraphics[width=0.26\textwidth]{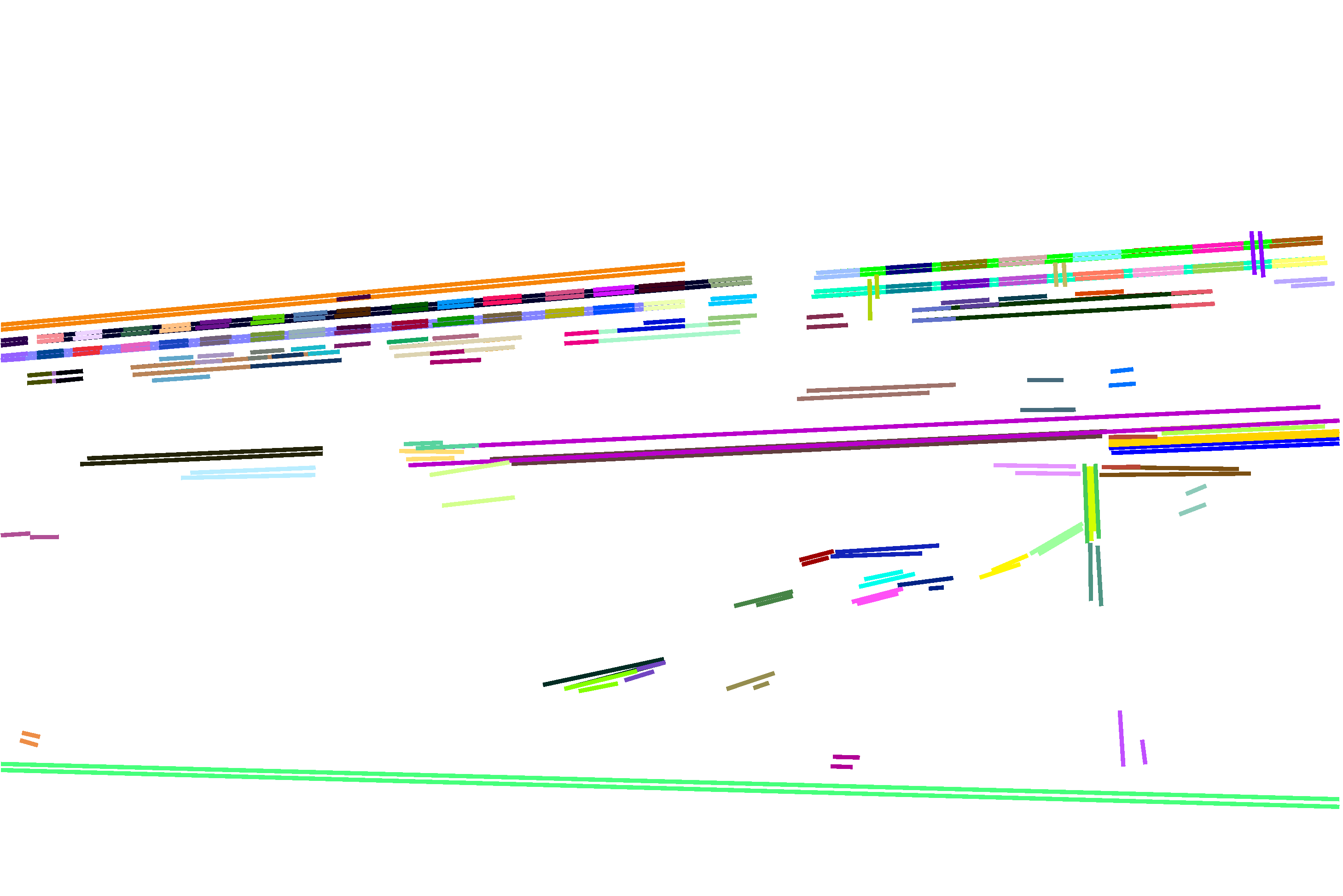}}~
\fbox{\includegraphics[width=0.26\textwidth]{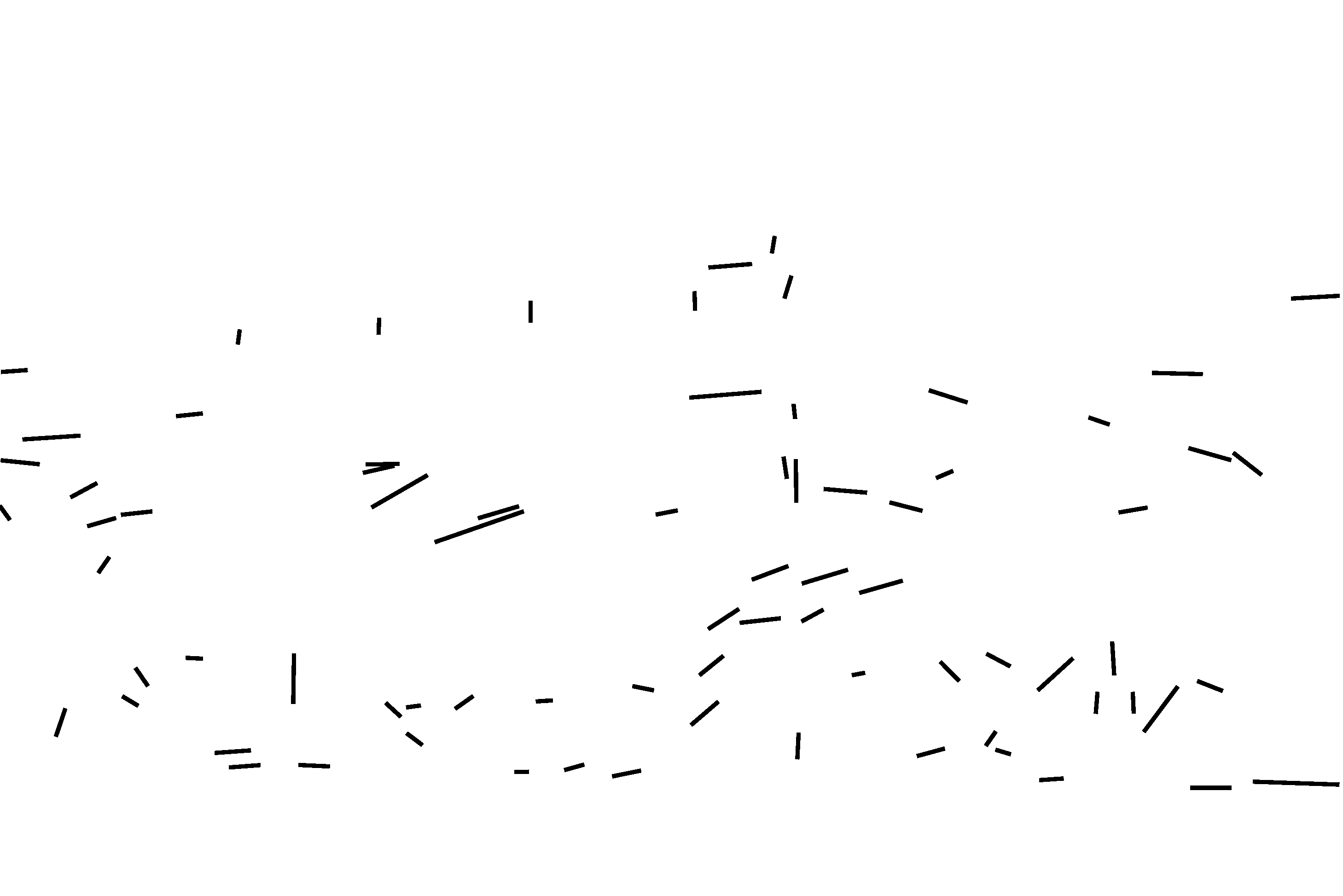}}
\caption{Gestalt detector performance over Whale image -- each color indicates
  one detected structure. Left to right and top to bottom: initial image, LSD
  line segments, all good continuations, non-local alignments, pairs of
  parallel segments (bars), and finally all segments that do not belong to any
  of the former structures. A segment can belong simultaneously to several of
  these higher order \textit{partial Gestalts}. Zoom-in by a factor of 400$\%$ is
  recommended.}
\label{fig:Detailed}
\vspace{-0.50cm}
\end{figure*}

\section{Conclusion}\label{sec: conclusion}

The three detectors proposed in this paper represent one step up in the Gestalt
grouping pyramid. The good experimental point is that few line segments are left
out unorganized, a requirement that was called ``articulation without rest" in
the Gestalt literature \cite{kanizsa1979organization}.  Clearly this step up
must be completed by further bottom up grouping. For example, good continuation
curves present gaps that must be completed with irregular contours. Other gaps
in good continuations or alignments must be explained by T-junctions; bars and
non-local alignments may be grouped again with the same good continuation and
parallelism principles.  These further steps are required to solve the
figure-background problem by unsupervised algorithms.
 
\section{Acknowledgment}
Work partly founded by the European Research Council (advanced grant Twelve Labours $n^\circ246961$).

\bibliography{refs}

\begin{figure} 
\centering
\fbox{\includegraphics[width=.3\columnwidth]{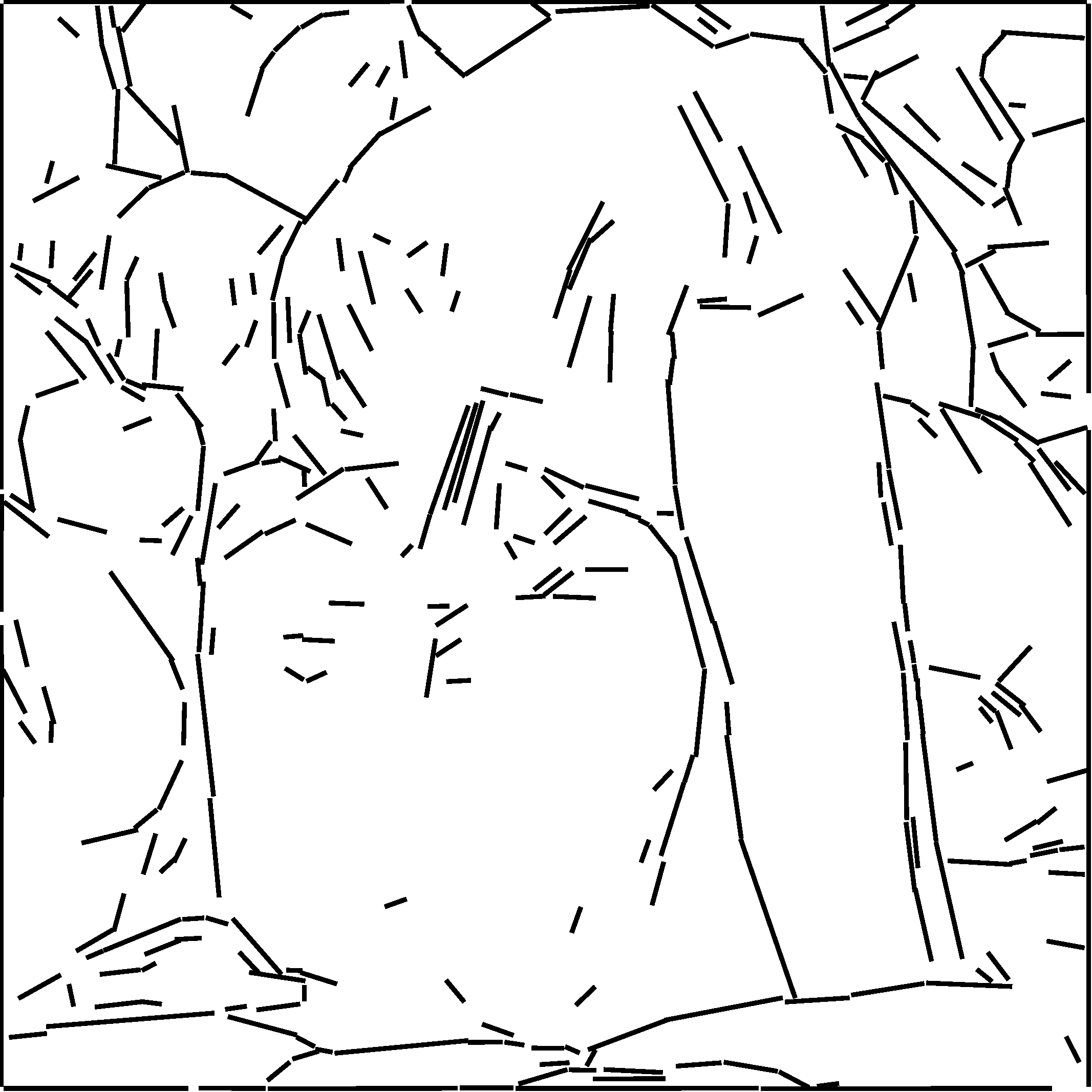}}~
\fbox{\includegraphics[width=.3\columnwidth]{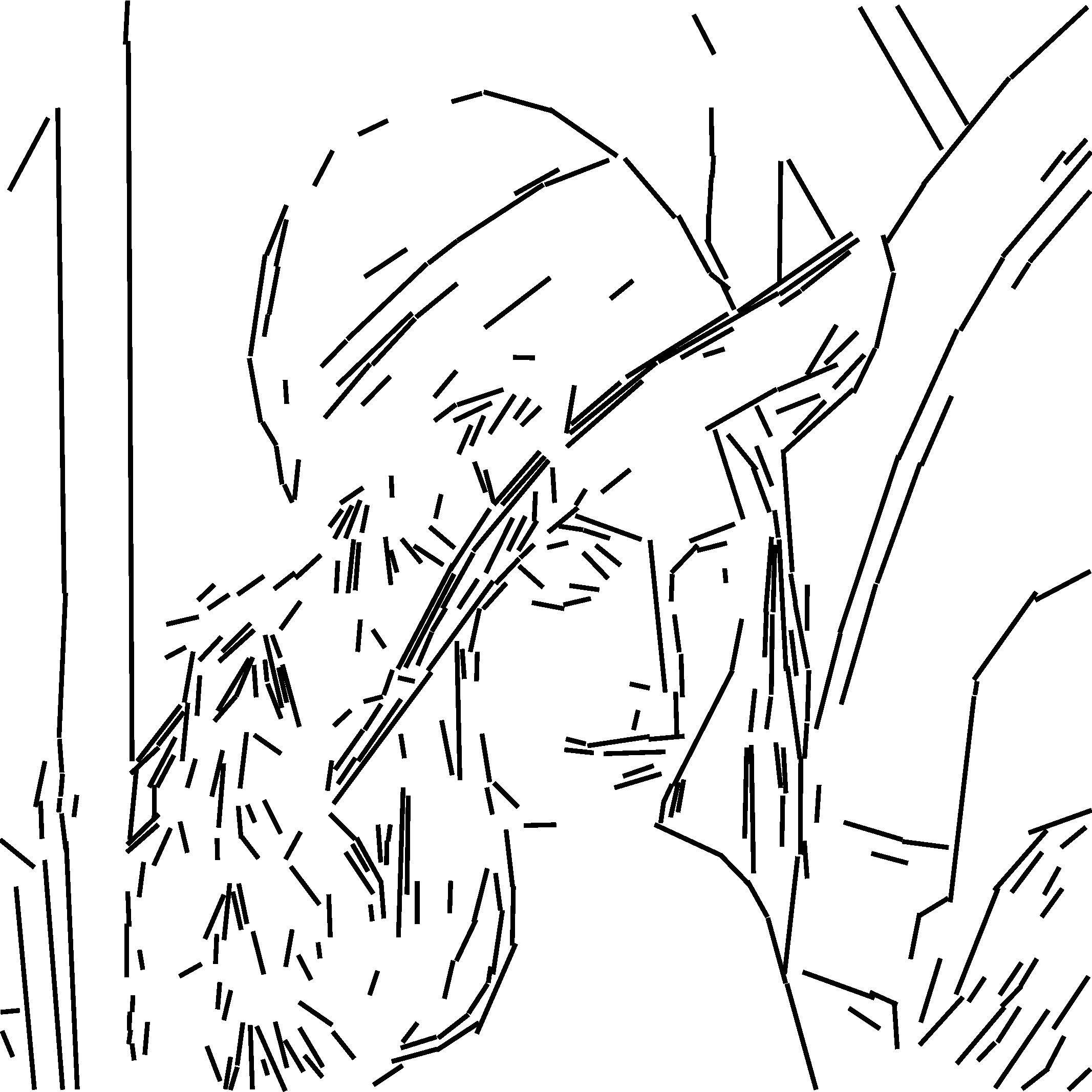}}~
\fbox{\includegraphics[width=.3\columnwidth , height=.3\columnwidth]{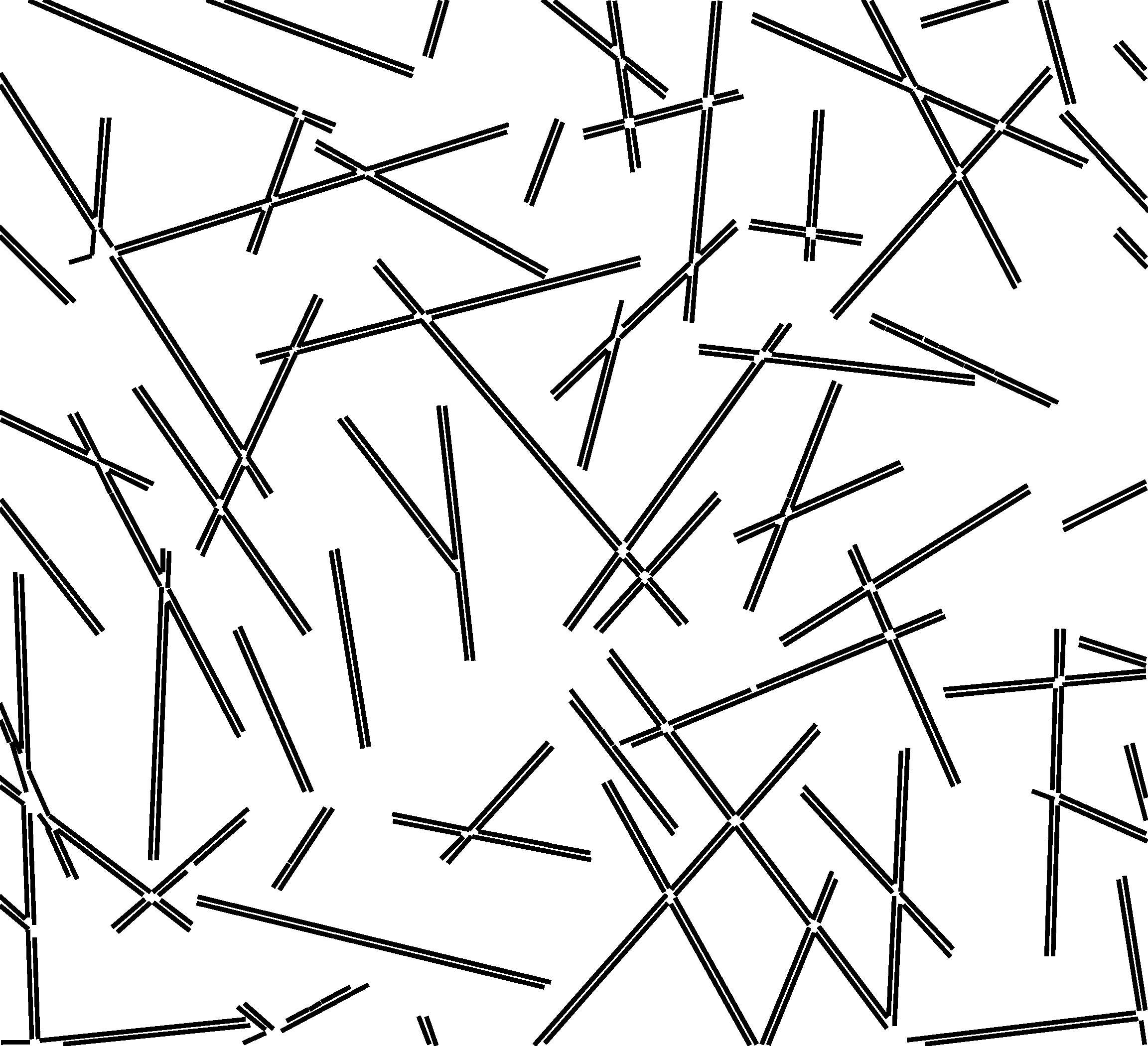}}
\fbox{\includegraphics[width=.3\columnwidth]{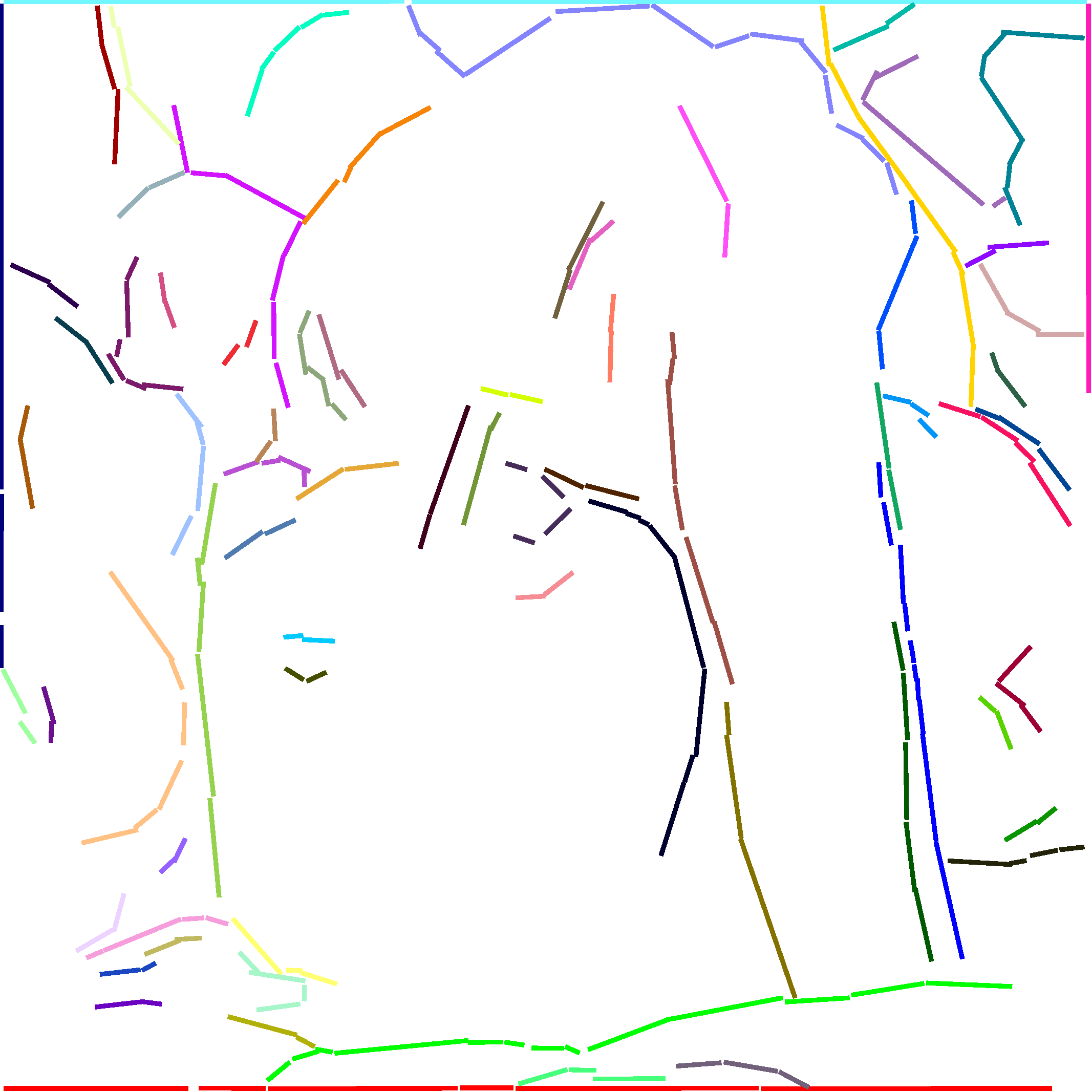}}~
\fbox{\includegraphics[width=.3\columnwidth]{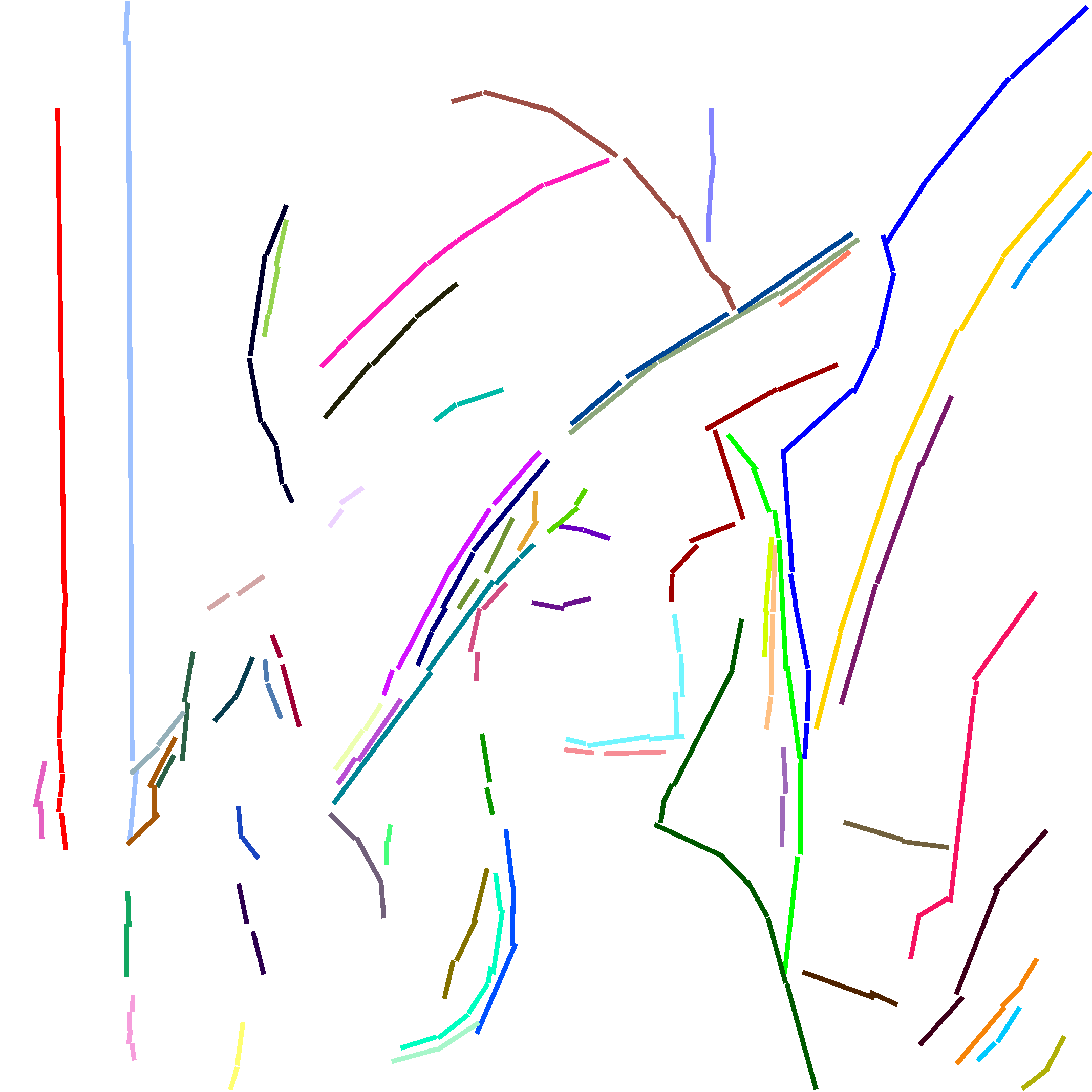}}~
\fbox{\includegraphics[width=.3\columnwidth , height=.3\columnwidth]{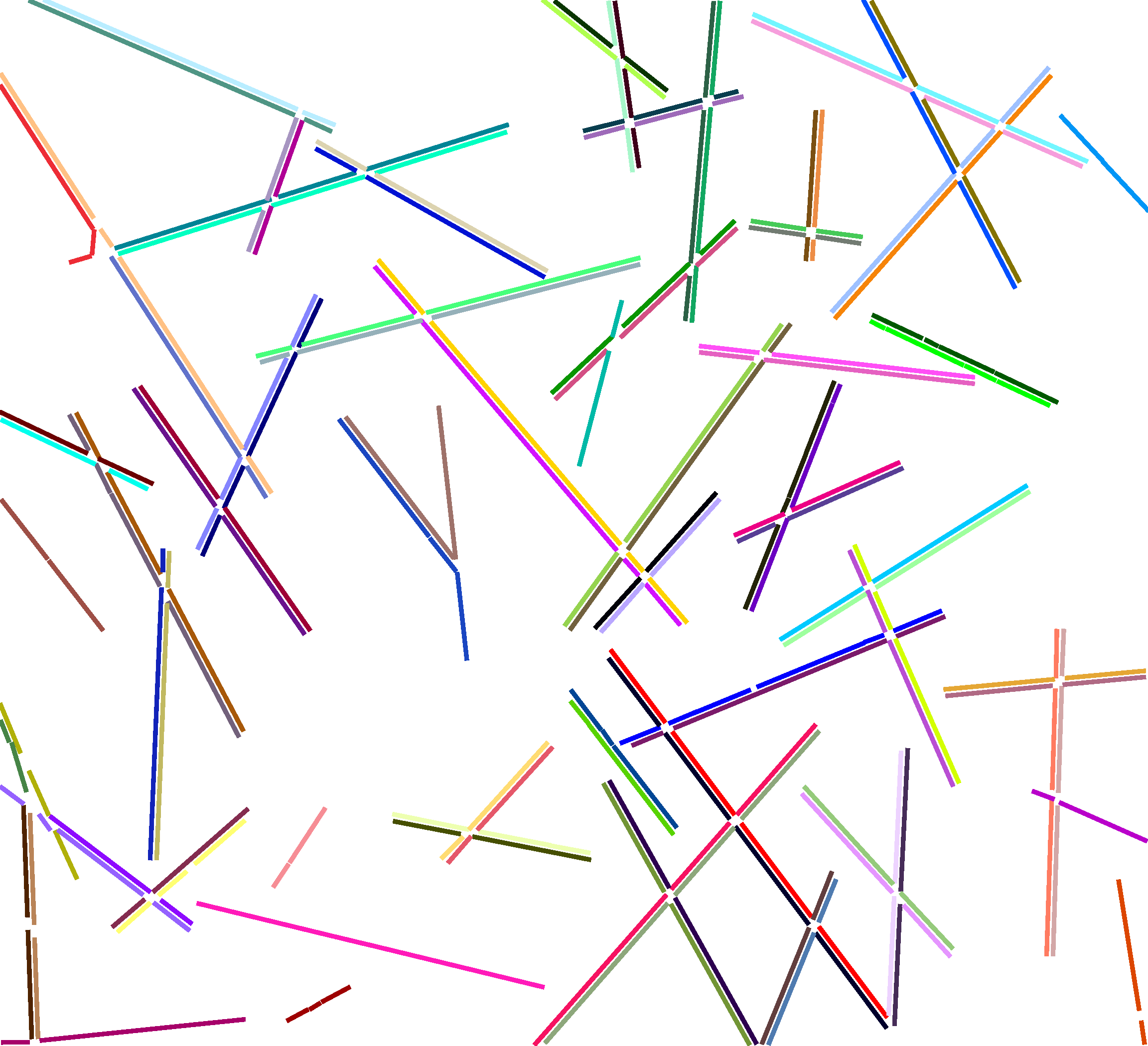}}
\caption{Good continuation grouping -- each color indicates one structure. First
  and second lines show LSD segments and good continuation segments,
  respectively, for (from left to right) Peppers, Lena and Random-Lines. In the
  Random-Line image as it is expected, no continuation is detected except over
  fragments of the same line. Zoom-in by a factor of 400$\%$ is recommended.}
\label{fig:Several}
\vspace{-0.50cm}
\end{figure}

\end{document}